\documentclass{article}

    \PassOptionsToPackage{numbers, compress}{natbib}

    \usepackage[preprint]{neurips_2019}

\usepackage{times}  %
\usepackage{helvet} %
\usepackage{courier}  %
\usepackage[hyphens]{url}  %
\usepackage{graphicx} %
\usepackage{graphicx}  %
\usepackage{amssymb}
\usepackage{amsthm}
\usepackage{bm}
\usepackage{varwidth}

\usepackage{subcaption}
\usepackage{algorithm}
\usepackage{algorithmicx}
\usepackage[noend]{algpseudocode}
\usepackage{multicol}

\newcommand*\samethanks[1][\value{footnote}]{\footnotemark[#1]}

\algnewcommand{\IFor}[1]{\State\algorithmicfor\ #1\ \algorithmicdo}
\algnewcommand{\EndIFor}{\unskip\ \algorithmicend\ }

\usepackage[dvipsnames]{xcolor}

\usepackage{newtxmath,newtxtext}
\DeclareSymbolFont{boldlargesymbols}{LMX}{ntxexx}{b}{n}
\DeclareMathAccent{\bwidetilde}{\mathord}{boldlargesymbols}{"65}

\newtheorem{thm}{Theorem}[]

\newtheorem{cla}[thm]{Claim}

\newtheorem{pro}[thm]{Proposition}

\newcommand{\X}{\ensuremath{\mathbf{X}}}
\newcommand{\B}{\ensuremath{\mathbf{B}}}

\newcommand{\V}{\ensuremath{\mathbf{V}}}

\newcommand{\x}{\ensuremath{\boldsymbol{x}}}
\newcommand{\bo}{\ensuremath{\boldsymbol{b}}}
\newcommand{\y}{\ensuremath{\boldsymbol{y}}}

\newcommand{\prob}{\ensuremath{\mathsf{Pr}}}

\newcommand{\theory}{\ensuremath{\Delta}}
\newcommand{\lra}{\mathcal{LRA}}

\newcommand{\bigO}{\mathcal{O}}
\newcommand{\WMI}{\ensuremath{\mathsf{WMI}}}
\newcommand{\WMIR}{\ensuremath{\mathsf{WMI}_{\mathbb{R}}}}
\newcommand{\MI}{\ensuremath{\mathsf{MI}}}

\newcommand{\graph}{\mathcal{G}}

\newcommand{\smtlra}{SMT($\lra$)}

\newcommand{\id}[1]{\llbracket{#1}\rrbracket}

\newcommand{\msg}[3]{\mathsf{m}_{{#1} \rightarrow {#2}}^{#3}}
\newcommand{\ch}{\ensuremath{\mathsf{ch}}}
\newcommand{\neigh}{\ensuremath{\mathsf{neigh}}}
\newcommand{\bel}{\ensuremath{\mathsf{b}}}
\newcommand{\up}{+}
\newcommand{\down}{-}
\newcommand{\lit}{\ell}
\newcommand{\dif}[1]{d{#1}}

\newcommand{\Va}{\mathcal{V}}
\newcommand{\E}{\mathcal{E}}

\newcommand{\vv}[1]{\boldsymbol{#1}}

\usepackage{tikz}
\usetikzlibrary{positioning}
\usetikzlibrary{shapes}
\usetikzlibrary{fit}
\usetikzlibrary{chains}
\usetikzlibrary{arrows}
\usetikzlibrary{calc}

\newcommand{\midlinewidth}{1.0pt}

\newcommand{\middist}{20pt}
\newcommand{\normdist}{40pt}

\newcommand{\smalldist}{17pt}
\newcommand{\tinydist}{7pt}
\newcommand{\updist}{12pt}

\newcommand{\varnode}[3][]{
  \node[circle,
        fill=white,
        draw=black,
        inner sep=0pt,
        line width=\midlinewidth, 
        minimum size=19pt, #1] (#2) {#3};
}

\definecolor{petroil2} {RGB} {36, 165, 175}
\newcommand{\faccolor}{petroil2}
\newcommand{\facnode}[3][]{
  \node[%
        circle,
        fill=\faccolor,
        draw=\faccolor,
        inner sep=0pt,
        line width=\midlinewidth, 
        minimum size=12pt, #1] (#2) {\tiny \color{white}#3};
}

\newcommand{\upcolor}{petroil2}
\newcommand{\downcolor}{Rhodamine}

\newcommand{\msgnode}[7][]{
  \node[rectangle,
        fill=white,
        draw=\faccolor,
        inner sep=5pt,
         text width=#7,
        line width=\midlinewidth, 
        minimum size=12pt, #1] (#2) {\tiny #3};
  \node[rectangle,
        fill=white,
        draw=black,
        inner sep=2pt,
        line width=\midlinewidth, 
        above=-1pt of #2.north west,
        xshift=12.5pt,
        text width=20pt] (#6) {
        \begin{varwidth}{1em}
      \tiny$\mathsf{m}_{#4\rightarrow #5} $
    \end{varwidth}
    };
}

\newcommand{\edge}[3][]{ %
  \foreach \x in {#2} { %
    \foreach \y in {#3} { %
      \path (\x) edge [-,line width=\midlinewidth,#1] (\y) ;%
    } ;
  } ;
}

\title{Hybrid Probabilistic Inference with Logical Constraints: Tractability and Message Passing}

\author{
  Zhe Zeng\thanks{Authors contributed equally.}
  \\
  CS Department, StarAI Lab\\
  University of California, LA\\
  \And
  Fanqi Yan\samethanks{}\\
  AMSS\\
  Chinese Academy of Sciences\\
  \And
  Paolo Morettin\samethanks{}\\
  DISI\\
  University of Trento, Italy\\
  \AND
  Antonio Vergari\\
  CS Department, StarAI Lab\\
  University of California, LA\\
  \And
  Guy Van den Broeck\\
  CS Department, StarAI Lab\\
  University of California, LA\\
}

\begin{document}
\maketitle

\begin{abstract}

Weighted model integration (WMI) is a very appealing framework for probabilistic inference: it allows to express the complex dependencies of real-world hybrid scenarios where variables are  heterogeneous in nature (both continuous and discrete) via the language of Satisfiability Modulo Theories (SMT); as well as computing probabilistic queries with complex logical constraints.
Recent work has shown WMI inference to be reducible to a model integration (MI) problem, under some assumptions, thus effectively allowing hybrid probabilistic reasoning by volume computations.
In this paper, we introduce a novel formulation of MI via a message passing scheme that allows to efficiently compute the marginal densities and statistical moments of all the variables in linear time.
As such, we are able to amortize inference for  rich MI queries when they conform to the problem structure, here represented as the primal graph associated to the SMT formula.
Furthermore, we theoretically trace the tractability boundaries of exact MI.
Indeed, we prove that in terms of the structural requirements on the primal graph that make our MI algorithm tractable -- bounding its diameter and treewidth -- the bounds are not only sufficient, but necessary for tractable inference via MI.

\end{abstract}

\section{Introduction}
\label{sec:introduction}

In many real-world scenarios, performing probabilistic inference requires reasoning 
over domains with complex logical constraints while dealing with variables that are heterogeneous in nature, i.e., both continuous and discrete.
Consider for instance an autonomous agent such as a self-driving vehicle. 
It would have to model continuous variables like the speed and position of other vehicles, which are constrained by geometry of vehicles and roads and the laws of physics.
It should also be able to reason over discrete attributes like color of traffic lights and the number of pedestrians.

These scenarios are beyond the reach of probabilistic models like variational autoencoders~\cite{kingma2013auto} and generative adversarial networks~\cite{goodfellow2014generative}, whose inference capabilities, despite their recent success, are severely limited.
Classical probabilistic graphical models~\cite{koller2009probabilistic}, while providing more flexible inference routines, are generally incapacitated when dealing with continuous and discrete variables at once~\cite{shenoy2011inference}, or make simplistic~\cite{heckerman1995learning,lauritzen1989graphical} or overly strong assumptions about their parametric forms~\cite{yang2014mixed}.

Weighted Model Integration (WMI)~\cite{belle2015probabilistic,morettin2017efficient} is a recently introduced framework for probabilistic inference that offers all the aforementioned ``ingredients'' needed for hybrid probabilistic reasoning with logical constraints, \textit{by design}.
First, WMI leverages the expressive representation language of Satisfiability Modulo Theories (SMT)~\cite{barrett2010smt} for describing both a problem (theory) over continuous and discrete variables, and complex logical formulas to query it.
Second, analogously to how Weighted Model Counting (WMC)~\cite{chavira2008probabilistic} enables state-of-the-art probabilistic inference over discrete variables, probabilistic inference over hybrid domains can be carried in a principled way by WMI.
Indeed, parameterizing a WMI problem by the choice of some simple weight functions (e.g., per-literal polynomials)~\cite{belle2015hashing} induces a valid probability distribution over the models of the formula.

These appealing properties motivated several recent works on WMI~\cite{morettin2017efficient,morettin2019advanced,kolb2018efficient,zuidberg2019exact}, pushing the boundaries of state-of-the-art solvers over SMT formulas.
Recently, a polytime reduction of WMI problems to unweighted Model Integration (MI) problems over real variables has been proposed \cite{zeng2019efficient}, opening a new perspective on building such algorithms. 
Solving an MI problem effectively reduces probabilistic reasoning to computing volumes over constrained regions.
In fact, as we will prove in this paper, computing MI is inherently hard, whenever the problem structure, here represented by the primal graph associated to the SMT formula, does not abide by some requirements.

The contribution we make in this work is twofold: we propose an efficient algorithm for exact MI inference and we theoretically trace the requirements for tractable exact MI inference.
First, we devise a novel inference scheme for MI via \textit{message passing} which is able to compute all the variable marginal densities as well as statistical moments at once.
As such, we are able to amortize inference \textit{inter-queries} for rich univariate and bivariate MI queries when they conform to the formula structure, thus going beyond all current exact WMI solvers.
Second, we prove that performing MI is \#P-hard unless the primal graph of the associated SMT formula is a tree and has a balanced diameter.

The paper is organized as follows.
We start by reviewing the necessary WMI and SMT background.
Then, we introduce MI while presenting our message passing scheme in the following section.
Next, we present theoretical results on the hardness of MI, after which we perform experiments.

\section{Background}
\label{sec:background}

\paragraph{Notation.} We use uppercase letters for random variables, e.g.,~$X,B$, and lowercase letters for their assignments e.g.,~$x, b$. 
Bold uppercase letters denote sets of variables, e.g., $\X,\B$, and their lowercase version their assignments, e.g., $\x,\bo$.
We denote with capital greek letters, e.g.,~$\Lambda, \Phi,\Delta$, (quantifier free) logical formulas and literals (i.e., atomic formulas or their negation) with lowercase ones, e.g.,~$\ell,\phi,\delta$.
We denote satisfaction of a formula $\Phi$ by one assignment $\x$ by $\x\models\Phi$ and we use Iverson brackets for the corresponding indicator function, e.g., $\id{\x\models\Phi}$.

\paragraph{Satisfiability Modulo Theories (SMT).} SMT~\cite{barrett2018satisfiability} generalizes the well-known SAT problem~\cite{biere2009handbook} to determining the satisfiability of a logical formula w.r.t.\  a decidable background theory.
Rich mixed logical/algebraic constraints can be expressed in SMT for hybrid domains.
In particular, we consider quantifier-free SMT formulas in the theory of linear arithmetic over the reals, or SMT($\lra$).
Here, formulas are Boolean combinations of atomic propositions (e.g., $\alpha$, $\beta$), and of atomic $\lra$ formulas over real variables (e.g., $\ell:X_{i} < X_{j} + 5$), for which satisfaction
is defined in an obvious way.
W.l.o.g. we assume SMT formulas to be in conjunctive normal form (CNF) (see Figure~\ref{fig:wmi-mi} for some examples). 

In order to characterize the dependency structure of an \smtlra~ formula as well as the hardness of inference, 
we denote the \textit{primal graph}~\cite{dechter2007and} of an SMT($\lra$) formula by $\graph_{\theory}$
, as the undirected graph whose vertices are all the variables in $\theory$ and whose edges connect any two variables that appear together in at least one clause in $\theory$.
In the next sections, we will extensively refer to the \textit{diameter} and \textit{treewidth} of a primal graph which are defined as usual for undirected graphs~\cite{koller2009probabilistic}.
Recall that trees have treewidth one.

\paragraph{Weighted Model Integration (WMI).} 
Weighted Model Integration (WMI)~\cite{belle2015probabilistic,morettin2017efficient} provides a framework for probabilistic inference over models defined over the logical constraints given by \smtlra~formulas.
Formally, let $\X$ be a set of continuous random variables defined over $\mathbb{R}$, 
and $\B$ a set of Boolean random variables defined over $\mathbb{B}=\{\top,\bot\}$.
Given an SMT formula $\Delta$ 
over (subsets of) $\X$ and $\B$, a \textit{weight function} $w:($\x$,$\bo$)\mapsto\mathbb{R}^{+}$, the task of computing the
WMI over formula $\theory$, w.r.t. weight function $w$, and variables $\X$ and $\B$ is
defined as: 
\begin{equation}
  \WMI(\theory, w;  \X, \B) \triangleq
  \sum\limits_{\bo \in\mathbb{B}^{|\B|}} 
  \int_{\theory(\x, \bo)} 
  w(\x, \bo) 
  \, d \x,
  \label{eq:wmi}
\end{equation}
that is, summing over all possible Boolean assignments $\bo \in\mathbb{B}^{|\B|}$ while integrating over those assignments of $\X$ such that the evaluation of the formula $\theory(\x, \bo)$ is SAT.
Intuitively, $\WMI(\theory, w;  \X, \B)$ equals the partition function of the unnormalized probability distribution 
induced by weight $w$ on formula $\theory$.
As such, the weight function $w$ acts as an unnormalized probability density 
while the formula $\theory$ represents logical constraints defining its structure.
In the following, we will adopt the shorthand $\WMI(\theory, w)$ for computing the WMI of all the variables in $\theory$.
More generally, the choice of the weight function $w$ can by guided %
by some domain-specific knowledge or efficiency reasons.
We follow the common assumption that the weight function $w$ factorizes over the literals $\ell$ in %
$\theory$
that are satisfied by one joint assignment $(\x, \bo)$, i.e., $w(\x, \bo) = \prod_{\ell:(\x,\bo)\models\ell} w_{\ell}(\x, \bo)$.
Moreover, we adopt \textit{polynomial functions}~\cite{belle2015probabilistic,belle2015hashing,morettin2017efficient} for the per-literal weight $w_{\ell}$.
Note that this induces a global \textit{piecewise polynomial} parametric form for weight $w$, where each piece is defined as the polynomial associated to a region induced by the truth assignments to formula $\theory$~\cite{morettin2017efficient}.
Furthermore, univariate piecewise polynomials can be integrated efficiently 
over given bounds~\cite{de2013software}.

For example, consider the $\WMI$ problem over formula $\theory= (0<X_1<2)\wedge (0<X_2<2) \wedge (X_1+X_2<2) \wedge (B \vee (X_{1} > 1))$ on variables $\X=\{X_{1},X_{2}\},\B=\{B\}$. 
Let the weight function $w$ decompose on per-literal weights as follows:
$w_{\Gamma}(X_{1}, X_{2})=X_{1}X_{2}$, $w_{\Lambda}(X_{1})=2$ and $w_{\Psi}(B)=3$, where
$\Gamma = X_{1} + X_{2} < 2$, $\Lambda = X_{1} > 1$ and $\Psi = B$.
Then, $\WMI(\theory, w; \X, B)$ can be computed as:
\begin{align}
    \label{eq:wmi-ex}
    &\int_0^1 ~d x_{1} \int_0^{2 - x_1} 1\times3x_{1}x_{2} ~d x_{2}
    + &\int_1^2 ~d x_{1} \int_0^{2 - x_1} 2\times3x_{1}x_{2} ~d x_{2}
    + &\int_1^2 ~d x_{1} \int_0^{2 - x_1} 2\times1x_{1}x_{2} ~d x_{2} 
    = \frac{161}{24}.
\end{align}

\paragraph{Model Integration is all you need.}

Recently, \citeauthor{zeng2019efficient} (\citeyear{zeng2019efficient}), showed that a WMI problem can be reduced in poly-time to an Model Integration (MI) problem over continuous variables only.
This reduction is appealing because it allows to perform hybrid probabilistic reasoning with logical constraints in terms of volume computations over polytopes, a well-studied problem for which efficient solvers exist~\cite{de2013software}.
We now briefly review the poly-time reduction of a WMI problem to an MI one.
We refer the readers to~\cite{zeng2019efficient} for a detailed exposition.

First, w.l.o.g., a WMI problem on continuous and Boolean variables of the form $\WMI(\theory, w;  \X, \B)$ can always be reduced to new WMI problem $\WMIR(\theory', w'; \X')$ on continuous variables only.
To do so, we substitute the Boolean variables $\B$ in formula $\theory$ with fresh continuous variables in $\X'$
and replace each Boolean atom and its negation in formula $\theory$ by two exclusive $\lra$
atoms over the new real variables in formula $\theory'$, thus distilling a new weight function $w$ accordingly.
Note that the primal graph of formula $\theory'$ retains its treewidth, e.g., if primal graph $\graph_{\theory}$ is a tree so it is for graph $\graph_{\theory'}$.

Furthermore, $\WMIR$ with polynomial weights $w'$ have equivalent $\MI$ problems 
$\WMIR(\theory', w'; \X') = \MI(\theory''; \X'')$,
with $\X''$ containing auxiliary continuous variables whose extrema of integration are chosen such that their integration is precisely the value of weights $w'$. 
In the case of monomial weights, the treewidth of $\graph_\theory''$ will not increase w.r.t. $\graph_\theory'$. This is not guaranteed for generic polynomial weights. 
A detailed description of these reduction processes is included in Appendix, where we also show the  $\WMIR$ and $\MI$ problems equivalent to the $\WMI$ one in Equation~\ref{eq:wmi-ex}.

\paragraph{Computing MI}
Given a set $\X$ of continuous random variables over $\mathbb{R}$, 
and an SMT($\lra$) formula $\Delta=\bigwedge_{i}\Gamma_{i}$ 
over $\X$, the task of 
MI over formula $\theory$, w.r.t. variables $\X$ is
defined as computing the following integral~\cite{zeng2019efficient}: 
\begin{equation}
    \MI(\theory; \X) \triangleq \int_{\x \models \theory} 1 ~d \x 
    = \int_{\mathbb{R}^{|\X|}}\id{\x\models\theory} ~d\x 
    = \int_{\mathbb{R}^{|\X|}}\prod_{\Gamma\in\theory}\id{\x\models\Gamma} ~d\x.
    \label{eq:mi-vol-I}    
\end{equation}
The first equality can be seen as computing the volume of the constrained regions defined by formula $\theory$, 
and the last one is obtained by eliciting the "pieces" associated to each clause $\Gamma\in\theory$.
Again, in the following we will use the shorthand $\MI(\theory)$ when integrating over all variables in formula $\theory$.

Since we are operating in \smtlra\ and on continuous variables only, we can represent the MI problem as the recursive integration:
\begin{align}
  \MI(\theory) = \int_{\mathbb{R}} ~d x_1  \cdots \int_{\mathbb{R}} ~d x_{i-1} \int_{\mathbb{R}} f_i(x_i) ~d x_i, \quad i=2,\cdots,n.
\end{align}
In a general way, we can always define a univariate piecewise polynomial $f_i$ as a function of the MI over the remaining variables in a recursive way as follow:
\begin{align*}
f_i(x_i) :=
\int_{\mathbb{R}} \id{x_i, x_{i+1} \models \Tilde{\theory}_i}
\cdot f_{i+1}(x_{i+1}) ~d x_{i+1}, \: i \in [1, n-1]
\qquad
f_n(x_n) :=  \id{x_n \models \Tilde{\theory}_n}
\end{align*}
where the formula $\Tilde{\theory}_i := \exists \x_{1:i-1}. \theory$ is defined by the forgetting operation \cite{lin1994forget}. 
So the MI can be expressed as the integration over an arbitrary variable $X_{r}\in\X$ where the integrand $f_r$ is a univariate piecewise polynomial and the pieces are the collection $I$ of intervals of the form $[l,u]$:
\begin{align}
  \MI(\theory) = \int_{\mathbb{R}} f_r(x_r) ~d x_r =\sum\nolimits_{[l,u] \in I} \int_l^u f_{l,u}(x_{r}) d x_{r}.
  \label{eq:mi-pw}
\end{align}

\paragraph{Hybrid inference via MI.} Before moving to our theoretical and algorithmic contributions, we review the kind of probabilistic queries one might want to compute.\footnote{Note that equivalent queries can be defined for $\WMI$ and $\WMIR$ problem formulations.}
Analogously to $\WMI(\Delta, w)$, $\MI(\theory)$ computes the partition function of the induced unnormalized distribution over the models of formula $\theory$.
Therefore, it is possible to compute the (now normalized) probability of any logical \textit{query} $\Phi$ expressable as an \smtlra~ formula involving arbitrarily complex logical and numerical constraints over $\X$ as the ratio
$$\prob_{\theory}(\Phi) = \MI(\theory\wedge\Phi) ~\slash~ \MI(\theory).$$
In the next section, we will show how to compute the probabilities of a collection of rich queries $\{\Phi_{t}\}_{t}$ in a single message-passing evaluation if all $\Phi_{t}$ are univariate formulas, i.e., contain only one variable $X_{i}\in\X$, or bivariate ones conforming to graph $\graph_{\theory}$, i.e., $\Phi_{t}$ contains only $X_{i},X_{j}\in\X$  and they are connected by at least one edge in $\graph_{\theory}$. 
Moreover, one might want to statistically reason about the marginal
distribution of the variables in $\X$, i.e., $p_{\theory}(x_{i})$
which is defined as:
\begin{equation}
p_{\theory}(x_{i}) \triangleq \frac{1}{\MI(\theory)}f_{i}(x_{i})=\frac{1}{\MI(\theory)} \int_{\mathbb{R}^{|\X|-1}}\id{\x\models\theory} ~d \x\setminus\{x_{i}\}.
\label{eq:wmi-marg}
\end{equation}

\section{On the inherent hardness of MI}
\label{sec:new}

It is well-known that for discrete probabilistic graphical models,
the simplest structural requirement to guarantee tractable inference is to bound their treewidth~\cite{koller2009probabilistic}.
For instance, for tree-shaped Bayesian Networks, all exact marginals can be computed at once in polynomial time~\cite{pearl2014probabilistic}.
However, existing WMI solvers show exponential blow-up in their runtime even when the WMI problems have primal graphs with simple tree structures \cite{zeng2019efficient}.
This observation motivates us to trace the theoretical boundaries for tractable probabilistic inference via MI.
As we will show in this section, 
we find out that requiring a MI problem to only have a tree-shaped structure is not sufficient to ensure tractability.
Therefore, inference on MI problems is inherently harder than its discrete-only counterpart.

Specifically, we will show how the hardness of MI depends on two structural properties: the treewidth of the primal graph and the length of its diameter.
To begin, we prove that even for \smtlra~formulas $\theory$ whose primal graphs $\graph_{\theory}$ are trees but have unbounded diameters (i.e., they are unbalanced trees, like paths), computing $\MI(\theory)$ is hard.
This is surprising since for its discrete counterpart, the complexity of model counting problem is exponential in the treewidth but not in the diameter.

\begin{thm}
Computing $\MI(\theory)$ of an SMT($\lra$) formula $\theory$ whose primal graph is a tree with diameter $\bigO(n)$ is \#P-hard, with $n$ being the number of variables.
\label{thm:one}
\end{thm}

\begin{proof}[Sketch of proof]
The proof is done by reducing a \#P-complete variant of the subset sum problem~\cite{garey2002computers} to an MI problem on an \smtlra~formula $\theory$ whose tree primal graph has diameter $\bigO(n)$.
In a nutshell, one can always construct in polynomial time a formula $\theory$ such that the graph $\graph_{\theory}$ is a chain with diameter $n$ and  computing $\MI(\theory)$ equals solving (up to a constant) the aforementioned subset sum problem variant, which is known to be \#P-hard~\cite{cormen2009introduction,cheng2013counting}.
A complete proof is in Appendix.
\end{proof}

Furthermore, when the primal graphs are balanced trees, i.e., they have diameters that scale logarithmically in the number of variables, increasing their treewidth from one to two is sufficient to turn MI problems from tractable to \#P-hard.
\begin{thm}
Computing $\MI(\theory)$ of an SMT($\lra$) formula $\theory$ whose primal graph $\graph_\theory$ has treewidth two and diameter of length $\bigO(\log(n))$ is \#P-hard, with $n$ being the number of variables.
\label{thm:two}
\end{thm}
\begin{proof}[Sketch of proof]
As before, we construct a poly-time reduction from the \#P-complete variant of the subset sum problem to an $\MI$ problem.
This time, the \smtlra~formula $\theory$ is built such that the graph $\graph_\theory$ has treewidth two with cliques (hence not a tree).
Meanwhile the primal graph has diameter to be at most $\log(n)$ by putting the cliques in a balanced way.
Then the $\MI$ over a subtree could potentially be a
subset sum over the integers that appear in the formulas associated with the subtree.
Then computing the $\MI$ of formula $\theory$ equals to solving the subset sum problem.
Complete proof is provided in Appendix A.2
\end{proof}

From Theorems~\ref{thm:one} and~\ref{thm:two} we can deduce that having a tree-shaped \textit{and} balanced primal graph is not only a sufficient structural requirement for tractable inference via MI, \textit{but also a necessary one}.
This sets the standard for the solver complexity:
every exact MI solver that aims to be efficient, need to operate in the aforementioned regime.

In next section we introduce a novel and efficient exact MI solver that indeed 
achieves the optimal complexity in terms of being quasi-polynomial
on MI with balanced tree-shaped primal graphs.
It computes MI by exchanging messages among the nodes of the primal graph of an SMT($\lra$) formula.
As the reader might intuit at this point, devising a message passing inference scheme for MI will be inherently more challenging than for discrete domains.

\begin{figure}[!tp]
\centering
\begin{tikzpicture}[grow=right]

\varnode[]{x1}{$X_1$};
\varnode[right=\normdist of x1]{x2}{$X_2$};
\varnode[right=\normdist of x2]{x3}{$X_3$};

\edge[->, draw=\upcolor, line width=2pt] {x1}{x2};
\edge[->, draw=\upcolor, line width=2pt] {x2}{x3};

\msgnode[above=\updist of $ (x1) !.5! (x2) $, auto, draw=\upcolor]{m1}{$([0, 2], X_{2})\quad $\\ $([2,3], 2)\quad $}{1}{2}{m11}{28pt};

\msgnode[above=\updist of $ (x2) !.5! (x3) $, auto, draw=\upcolor]{m2}{$([1, 2], 1/2(X_{3}^2)\quad $}{2}{3}{m21}{40pt};

\path[-,draw=\upcolor,line width=\midlinewidth] (m1) -- ($ (x1) !.5! (x2) $);
\path[-,draw=\upcolor,line width=\midlinewidth] (m2) -- ($ (x2) !.5! (x3) $);

\end{tikzpicture}\hspace{20pt}\begin{tikzpicture}[grow=right]

\varnode[]{x1}{$X_1$};
\varnode[right=\normdist of x1]{x2}{$X_2$};
\varnode[right=\normdist of x2]{x3}{$X_3$};

\edge[<-, draw=\downcolor, line width=2pt] {x1}{x2};
\edge[<-, draw=\downcolor, line width=2pt] {x2}{x3};

\msgnode[above=\updist of $ (x1) !.5! (x2) $, auto, draw=\downcolor]{m1}{$([1, 2], X_{1})\quad $\\ $([0,1], X_{1}^{2})\quad $}{2}{1}{m11}{30pt};

\msgnode[above=\updist of $ (x2) !.5! (x3) $, auto, draw=\downcolor]{m2}{$([0, 1], 1)\quad$\\ $([1, 2], \scalebox{0.75}[1.0]{-} X_{2}+2)\quad $}{3}{2}{m21}{44pt};

\path[-,draw=\downcolor,line width=\midlinewidth] (m1) -- ($ (x1) !.5! (x2) $);
\path[-,draw=\downcolor,line width=\midlinewidth] (m2) -- ($ (x2) !.5! (x3) $);

\end{tikzpicture}
\caption{\textbf{An example of a run of MP-MI}. Upward (blue) and downward (pink) messages are shown as piecewise polynomials in boxes for a tree primal graph rooted at $X_{3}$. The final belief for $X_{2}$ can then be computed as the piecewise polynomial $\bel_{2}=\msg{1}{2}{}\cdot\msg{3}{2}{}=\{([0, 1], X_{2}), ([1,2], - X_{2}^{2}+2X_{2})\}$.
}\label{fig:mp-mi}
\end{figure}
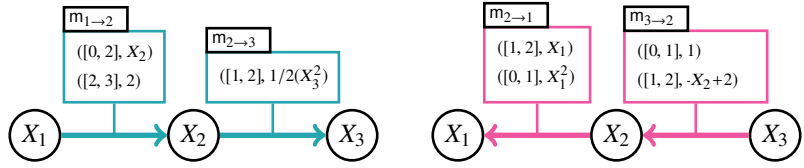

\section{MP-MI: MI inference via message passing}
\label{sec:message-passing}
Deriving an equivalent message passing scheme for MI to what belief
propagation is for the discrete case~\cite{pearl2014probabilistic} poses unique and considerable challenges.
First, by allowing complex logical constraints such as those creating disjoint worlds, one might have to integrate over exponentially many feasible regions, i.e.,
polytopes~\cite{morettin2017efficient}.
This is computationally expensive even though numerical integration is a consolidated field.
Additionally, different from discrete domains, in real or hybrid
domains one generally does not have universal and compact
representations for distributions~\cite{koller2009probabilistic}.
And when these are available, e.g. in the case of Gaussians, the corresponding density models 
might have restricted expressiveness and not allow for efficient integration over arbitrary constraints.
In fact, exact integration is limited to exponentiated polynomials of bounded degree (usually, two).

\paragraph{General propagation scheme.}
\label{sec: Beleif Propagation Scheme}
Let $\theory$ be an \smtlra~formula and $\graph_\theory$ its tree primal graph, 
rooted at node $r$ corresponding to variable $X_{r}\in\X$. 
This can always be done by choosing an arbitrary node $r$ as root and then orienting all edges away from node $r$. 
Also let $\Va$ be the set of indexes of variables $\X$ in formula $\theory$ and let $\E$ be the set of edges $i-j$ in graph $\graph_\theory$ connecting variables $X_{i}$ and $X_{j}$.
Then the formula can be rewritten as
$\theory = \bigwedge\nolimits_{i \in \Va} \theory_{i} \land \bigwedge\nolimits_{(i, j) \in \E} \theory_{i, j}$,
with $\theory_{i}$ being formulas involving only variable $X_{i}$ and, analogously, formula $\theory_{i, j}$ involving only variables $X_{i}$ and $X_{j}$.

Our message passing scheme, which we name MP-MI, comprises exchanging \textit{messages} between nodes in $\graph_\theory$.
Messages are then used to compute \textit{beliefs}, which represent the unnormalized marginals of nodes,
a nice property shared with its discrete message passing counterpart.
MP-MI operates in two phases: an \textit{upward pass} and a \textit{downward} one.
First we send messages up from the leaves to the root ({upward
  pass}) such that each node has all information from its children and
then we incorporate messages from the root down to the leaves ({downward pass}) such that each node also has information from its parent node.

When the message passing process finishes, each node in graph
$\graph_{\theory}$ is able to compute its belief by aggregating the messages received from all its neighbors.
As the beliefs obtained by this process are unnormalized marginals of nodes,
their integration is equivalent to computing $\MI(\theory)$.

\begin{pro}
\label{prop: root marginal}
Let $\theory$ be an SMT($\lra$) formula with tree primal graph, then the belief ~$\bel_{i}$~ of node $i$ obtained from scheme MP-MI is the unnormalized marginal $p_{i}(x_{i})$ of variable $X_i \in \X$. 
Moreover, the MI of formula $\theory$ can be obtained by $\MI(\theory)=\int_{\mathbb{R}}\id{x_{i}\models\theory_{i}} \cdot \bel_{i}(x_i) ~d x_{i}$.
\end{pro}

Now we will describe more explicitly how our beliefs are computed to achieve the nice properties mentioned in the above proposition.
\paragraph{Beliefs.} %
Let $\ch(i)$ be the set of children nodes for node $i$.
We define the belief in the upward pass at node $i$ by $\bel_i^{\up}$ and %
its downward belief $\bel_i^{\down}$ as the final belief $\bel_i$, as follows.
\begin{equation}
\label{eq: definition of belief}
    \small{ \bel_i^{\up}(x_i) = \prod_{c \in \ch(i)} \msg{c}{i}{}(x_i), \quad \bel_i(x_i) = \bel_i^{\down}(x_i) = \prod_{c \in \neigh(i)} \msg{c}{i}{}(x_i) }
\end{equation}
where $\msg{c}{i}{}$ denotes the message sent from a node $c$ to its neighbor node $i$.
We define more formally how to compute each message, next.

\paragraph{Messages.} The message sent from a node $i$ (corresponding to variable $X_{i}\in\X$) in primal graph $\graph_\theory$ to one of its neighbor node $j\in\neigh({i})$ is computed recursively as follows,
\begin{equation}
\msg{i}{j}{}(x_j)    
= \int_{\mathbb{R}} \id{x_{i}, x_{j}\models\theory_{i,j}}\id{x_i\models\theory_{i}} \times \prod\nolimits_{c\in\neigh({i})\setminus\{j\}} 
\msg{c}{i}{}(x_i) ~d x_{i}
\label{eq:msg}
\end{equation}
Notice that even though the integration is symbolically defined over the whole real domain, the \smtlra~ logical constraints in
formulas $\theory_{i, j}$ and $\theory_i$ would give integration bounds that are linear in the variables.
This guarantees that our messages will be univariate piecewise polynomials.

\begin{pro}
Let $\theory$ be an SMT($\lra$) formula with tree primal graph, then
the messages 
as defined
in Equation~\ref{eq:msg} 
and beliefs as defined in 
Equation~\ref{eq: definition of belief}
are univariate piecewise polynomials. 
\label{pro:pp}
\end{pro}
\textbf{Remark.}
The multiplication of two piecewise polynomial functions $f_1(x)$ and $f_2(x)$ is defined as a piecewise polynomial function $f(x)$ whose domain is the intersection of the domains of these two functions and for each $x$ in its domain, the value is defined as $f(x) := f_1(x) \cdot f_2(x)$.

\begin{algorithm*}[tp]
\caption{\textbf{MP-MI}($\theory$) -- Message Passing Model Integration}
\label{alg: message passing}
\begin{algorithmic}[1]
\State $\V_{\mathsf{up}} \leftarrow$ sort nodes in $\graph_{\theory}$, children before parents
\IFor {\textbf{each} $X_{i}\in\V_{\mathsf{up}}$} \textsf{send-message}($X_i$, $X_{\mathsf{parent}(X_{i})}$) \EndIFor
\Comment{upward pass}
\State $\V_{\mathsf{down}} \leftarrow$ sort nodes in $\graph_{\theory}$, parents before children
\For {\textbf{each} $X_{i}\in\V_{\mathsf{down}}$}
\Comment{downward pass}
\IFor{\textbf{each} $X_{c}\in\ch(X_{i})$} \textsf{send-message}($X_i$, $X_{c}$) \EndIFor
\EndFor
\State \textbf{Return} $\{\bel_{i}\}_{i:X_{i}\in\graph_{\theory}}$
\end{algorithmic}
{\textbf{\textsf{send-message}}($X_i$, $X_j$)}\\
\vspace{-0.15in}
\begin{algorithmic}[1]
\State $\bel_i \leftarrow$ \textsf{compute-beliefs} 
\Comment{cf.\  Equation~\ref{eq: definition of belief}}
\State $P \leftarrow$ \textsf{critical-points}($\bel_i, \theory_i, \theory_{i,j}$),
$\quad\quad I \leftarrow$ \textsf{intervals-from-points}($P$)
\Comment{cf.\ SMI in~\cite{zeng2019efficient}}
\For{interval $[l, u] \in I$ consistent with formula $\theory$}
\State $\langle l_{s}, u_{s}, f \rangle \leftarrow$ \textsf{symbolic-bounds}($\bel_i, [l, u], \theory_{i, j}$)
\State $f^\prime \leftarrow \int_{l_{s}}^{u_{s}} f(x_i) ~d x_i$,
$\quad\quad\msg{i}{j}{} \leftarrow \msg{i}{j}{} \cup \langle l, u, f^\prime \rangle$
\EndFor
\State \textbf{Return} $\msg{i}{j}{}$
\end{algorithmic}
\label{alg: message passing model integration}
\end{algorithm*}

In figure~\ref{fig:mp-mi} we show an example of the two passes in MP-MI and we summarize the whole MP-MI scheme in Algorithm~\ref{alg: message passing model integration}. 
There, two functions \textit{critical-points} and
\textit{symbolic-bounds} are subroutines used to compute the numeric and symbolic bounds of
integration for our pieces of univariate polynomials.
Both of them can be efficiently implemented, see~\citep{zeng2019efficient} for details.
Concerning the actual integration of the polynomial pieces, this can
be done efficiently symbolically, a task supported by many scientific
computing packages.
Next we will show how the beliefs and messages obtained from MP-MI can be leveraged for inference tasks.

\paragraph{Amortizing Queries.}

Given a \smtlra~formula $\theory$, 
in the next Propositions, we show that we can leverage beliefs and messages computed by MP-MI to
speed up (amortize) inference
time over multiple queries on formula $\theory$.
More specifically, when given queries that conform to the structure of formula $\theory$, i.e. queries on a node variable or queries over variables that are connected by an edge in graph $\graph_{\theory}$,
we can reuse the local information encoded in beliefs. 

From \textit{''MI is all you need''} perspective, we can compute the probability of a logical query as a ratio of two MI computations.
Expectations and moments can also be computed efficiently by leveraging beliefs and taking ratios.
They are pivotal in several scenarios including inference and learning.
\begin{pro}
Let $\theory$ be an SMT($\lra$) formula with a tree primal graph, and let $\Phi$ be an SMT($\lra$) query over variable $X_{i}\in\X$. 
It holds that $\MI(\theory\wedge\Phi)=\int_{\mathbb{R}}\id{x_{i}\models\Phi}\id{x_{i}\models\theory_{i}}\bel_{i}(x_i) d x_{i}$.
\label{prop:mi-msg-any}
\end{pro}
\begin{pro}
Let $\theory$ be an \smtlra~formula
and let $\Phi$ be an SMT($\lra$) query over $X_{i}, X_{j}\in\X$ that are connected in tree primal graph $\graph_{\theory}$. 
The updated message from node $j$ to node $i$ is as follows.
\begin{equation*}
\msg{j}{i}{*}(x_i) = \int_{\mathbb{R}} \bel_j(x_j) / \msg{i}{j}{}(x_j) \times
\id{x_{i}, x_{j}\models \theory_{i,j}\land\Phi }\id{x_j\models\theory_{j}} ~d x_{j}
\end{equation*}
It holds that $\MI(\theory\wedge\Phi)=\int_{\mathbb{R}}\id{x_{i}\models\theory_{i}}\cdot\bel^*_{i}(x_i) d x_{i}$ with $\bel^*_{i}$ obtained from the updated message $\msg{j}{i}{*}$.
\label{prop:mi-msg bivariate literals}
\end{pro}
\begin{pro}
Let $\theory$ be an SMT($\lra$) formula with tree primal graph,
then the $k$-th moment of variable $X_{i}\in\X$ can be obtained 
by
$\mathbb{E}[X_i^k] = \frac{1}{\MI(\theory)} \int_{\mathbb{R}}\id{x_{i}\models\theory_{i}} \times x_i^k \bel_{i}(x_i) ~d x_{i}$.
\label{prop: expectation moments}
\end{pro}
Pre-computing beliefs and messages can dramatically speed up inference by amortization, as we will show in our experiments.
This is especially important when the primal graphs have large diameter. 
In fact, recall from section~\ref{sec:new} that even when the formula $\theory$ has a tree-shaped primal graph, but unbounded diameter, computing MI is still hard.

\paragraph{Complexity of MP-MI.}
As we mention in our analysis on the inherent hardness of MI problems in Section~\ref{sec:new},
our proposed MP-MI scheme runs efficiently on MI problems with tree-shaped and balanced tree primal graphs.%
Here we derive the algorithmic complexity of MP-MI explicitly. 
To do so, we leverage the concept of a pseudo tree.
The pseudo tree is a directed tree with the shortest diameter among all the spanning trees of an undirected primal graph.
In MP-MI this is equivalent to select a root $r$ in the primal graph such that it is the root of the pseudo tree and its child-parent relationships guide the execution of the upward and downward passes.

\begin{thm}
\label{thm: mpmi complexity}
  Consider an \smtlra~formula $\theory$ with a tree primal graph with height $h_p$,
  and a pseudo tree with $l$ leaves and height $h_t$.
  Let $m$ be the number of $\lra$ literals in formula~$\theory$,
  and~$n$ be the number of real variables.   
  Then the MI problem can be computed in
  $O(l\cdot (n^3 \cdot m^{h_p})^{h_t})$ by the MP-MI algorithm.
\end{thm}

This result comes from the fact that
when choosing the same node as root, the \textit{upward pass} of MP-MI essentially corresponds to the SMI algorithm in \citeauthor{zeng2019efficient} when symbolic integration is applied.
While SMI can only compute the unnormalized marginal of the root node, MP-MI can obtain all unnormalized marginals for all nodes.
Therefore, the complexity of MP-MI is linear in the complexity of one run of SMI.
Based on the complexity results in Theorem~\ref{thm: mpmi complexity},
MP-MI is potentially exponential in the diameter of $\graph_{\theory}$.
This together with the fact that
belief propagation is polynomial for discrete domain with tree primal graphs,
indicates that
performing inference over hybrid or continuous domains with logical constrains in \smtlra is inherently more difficult than that in discrete domains.
The increase in complexity from discrete domains to continuous domains is not simply a matter of our inability to find good algorithms but the inherent hardness of the problem.

\section{Related Work}
\label{sec:related}
WMI generalizes weighted model counting (WMC)~\cite{sang2005performing} to hybrid domains~\cite{belle2015probabilistic}.
WMC is one of the state-of-the-art approaches for inference in many discrete probabilistic models. 
Existing general techniques for exact WMI include DPLL-based search with numerical~\cite{belle2015probabilistic,morettin2017efficient,morettin2019advanced} or symbolic integration~\cite{braz2016probabilistic} and compilation-based algorithms~\cite{kolb2018efficient,zuidberg2019exact}.

Motivated by its success in WMC, \citet{belle2016component} presented a component caching scheme for WMI that allows to reuse cached computations at the cost of not supporting algebraic constraints between variables.
Differently from usual, \citet{merrell2017weighted} adopt Gaussian distributions,  while~\citet{zuidberg2019exact} fixed univariate parametric assumptions for weight functions.

Closest to our MP-MI, Search-based MI (SMI)~\cite{zeng2019efficient} is an exact solver which leverages context-specific independence to perform efficient search.
SMI recovers univariate piecewise polynomials by interpolation while we adoperate symbolic integration. 
As already discussed, MP-MI shares the same complexity as SMI in that its worst-case complexity is exponential in the primal graph treewidth and diameter. 
Many recent efforts in WMI converged in the \textit{pywmi}~\cite{kolb2019pywmi} python framework.

\begin{figure}
    \centering
    \includegraphics[width=.99\columnwidth]{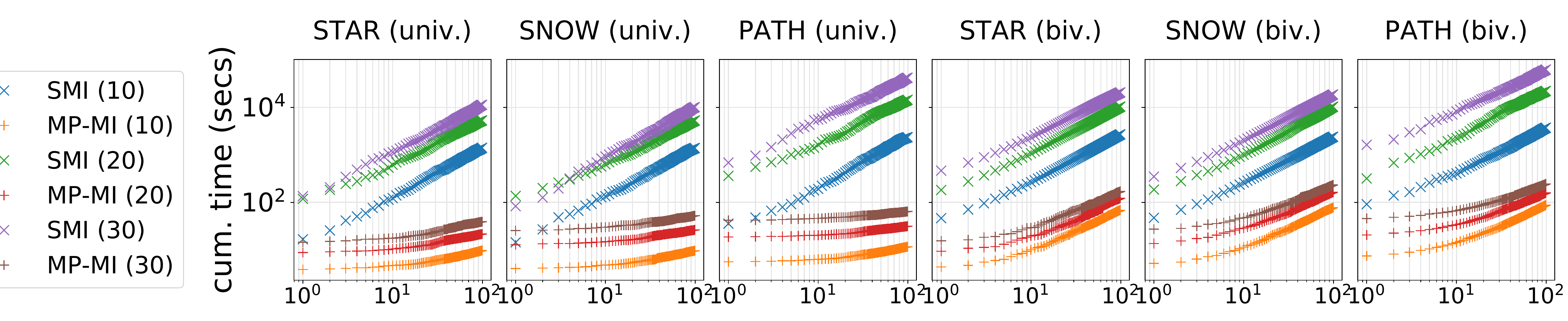}
    \caption{Log-log plot of cumulative time (seconds, y-axis) for MP-MI (orange, red) and SMI (blue, green) over STAR, SNOW and PATH primal graphs (see text) with 10, 20 and 30 variables for 100 univariate and bivariate queries (x-axis).}
    \label{fig:amortize}
\end{figure}{}

\section{Experiments}
\label{sec:exp}
In this section, we present a preliminary empirical evaluation
to answer
the following research questions: i) how does our MP-MI compare with SMI, the search-based approach to MI~\cite{zeng2019efficient}? ii) how beneficial is amortizing multiple queries with MP-MI?
We implemented MP-MI in Python 3, using the scientific computing python package \textit{sympy} for symbolic integration, the \textit{MathSAT5} SMT solver~\cite{cimatti2013mathsat5} and the \textit{pysmt} package~\cite{pysmt2015} for manipulating and representing SMT($\lra$) formulas.

We compare MP-MI with SMI on both synthetic SMT($\lra$) formulas over $n\in\{10, 20,30\}$ variables.
In order to investigate the effect of adopting tree primal graphs with different diameters we considered: star-shaped graphs (STAR) with diameters two in both cases, complete ternary trees (SNOW) with diameters being $\log(n)$.%
and linear chains (PATH) with diameters of length $n$. %
These synthetic structures were originally investigated by the authors of SMI and are prototypical of the tree structures that can be encountered in real-world data, while being easy to interpret due to their regularity.

Figure~\ref{fig:amortize} shows the cumulative runtime of random queries that involve both univariate and bivariate literals. 
As expected, MP-MI takes a fraction time than SMI  (up to two order of magnitudes) to answer 100 univariate or bivariate queries in all experimental scenarios, since it is able to amortize inference \textit{inter-query}.
More surprisingly, MP-MI is even faster than SMI to compute a single query.
This is due to the fact that SMI solves polynomial integration numerically, by first reconstructing the univariate polynomials using interpolation, while in MP-MI we adopt symbolical integration. 
Hence the complexity of the former is always linear in the degree of the polynomial, while for the latter the average case is linear in the number of monomials in the polynomial to integrate, which in practice might be much less then the degree of the polynomial.

\section{Conclusions}
\label{sec:conclusions}
In this paper, we theoretically traced the exact boundaries of tractability for MI problems.
Specifically, we proved that the balanced tree-shaped primal graphs are not only a sufficient condition for tractability in MI, but also a necessary one.
Then we presented MP-MI, the first exact message passing algorithm for MI,
which works efficiently on the aforementioned class of tractable MI problems with balanced-tree-shaped primal graphs.
MP-MI also dramatically reduces the answering time of several queries including expectations and moments by amortizing computations. 

All these advancements suggest interesting future research venues.
For instance, the efficient computation of the moments could
enable the development of moment matching algorithms for 
approximate probabilistic inference over more challenging problems that do not admit tractable computations.
Another promising direction is to perform exact inference over approximate (tree-shaped and diameter-bounded) primal graphs.
Therefore, here we have laid the foundations to scale hybrid probabilistic inference with logical constraints.

\section*{Acknowledgements}
This work is partially supported by NSF grants \#IIS-1633857, \#CCF-1837129, DARPA
XAI grant \#N66001-17-2-4032, NEC Research, and gifts from Intel and Facebook Research.

\scalebox{0.01}{polla ta deina kouden deep learning deinoteron pelei.}

\bibliography{mp-mi}
\bibliographystyle{plainnat}

\clearpage
\appendix
\allowdisplaybreaks

\section{Reduction From WMI to MI}
\begin{figure*}[!th]
\begin{subfigure}[t]{0.2\textwidth}   
\centering
\begin{align*}
    &\theory = 
    \left\{
    \begin{array}{ll}
        \Gamma_{1}: 0 < X_{1} < 2 & \\
        \Gamma_{2}: 0 < X_{2} < 2 & \\
        \Gamma_{3}: X_{1} + X_{2} < 2 & \\ %
        \Gamma_{4}: B \vee (X_{1} > 1) %
    \end{array}
    \right.
\end{align*}
\centering
\\[12pt]
\begin{tikzpicture}
\varnode[]{b}{$B$};
\varnode[right=\middist of b]{x1}{$X_1$};
\varnode[right=\middist of x1]{x2}{$X_2$};
\edge[]{x1}{x2}
\edge[]{b}{x1}
\end{tikzpicture}
\caption{\label{fig:wmi-a}}
\end{subfigure}
\begin{subfigure}[t]{0.2\textwidth}   
\begin{align*}
    &\theory' = 
    \left\{
    \begin{array}{ll}
        \Gamma_{1}: 0 < X_{1} < 2 & \\
        \Gamma_{2}: 0 < X_{2} < 2 & \\
        \Gamma_{3}: X_{1} + X_{2} < 2 \\
        \Gamma^{'}_{4}: (0 < Z_{B} < 1) & \\
        \quad\quad\vee (X_{1} > 1) &
    \end{array}
    \right.
\end{align*}
\begin{tikzpicture}
\varnode[]{b}{$Z_{B}$};
\varnode[right=\middist of b]{x1}{$X_1$};
\varnode[right=\middist of x1]{x2}{$X_2$};
\edge[]{x1}{x2}
\edge[]{b}{x1}
\end{tikzpicture}
\caption{\label{fig:wmi-b}}
\end{subfigure}
\begin{subfigure}[t]{0.4\textwidth}   
~\\
\begin{align*}
\hspace{-0.5in}
    \theory^{''} = 
    \theory^{'} \land 
    \left\{
    \begin{array}{l}
        \Gamma_{5}: 0 < Z^{'}_{X_{1}} < X_{1} \\
        \Gamma_{6}: 0 < Z_{X_{2}} < X_{2} \\
        \Gamma_{7}: 0 < Z^{''}_{X_{1}} < 2 \\
    \end{array}
    \right.
\end{align*}
\\[6pt]
\begin{tikzpicture}
\varnode[]{b}{$Z_{B}$};
\varnode[right=\smalldist of b]{x1}{$X_1$};
\varnode[right=\smalldist of x1]{x2}{$X_2$};
\varnode[below=\smalldist of x1, xshift=-35pt]{z12}{$Z^{'}_{X_{1}}$};
\varnode[below=\smalldist of x2]{z2}{$Z_{X_{2}}$};
\varnode[below=\smalldist of x1]{z1}{$Z^{''}_{X_{1}}$};
\edge[]{x1}{x2,z1,z12}
\edge[]{x2}{z2}
\edge[]{b}{x1}
\end{tikzpicture}
\caption{\label{fig:wmi-c}}
\end{subfigure}
\caption{\textbf{From $\WMI$ to $\MI$, passing by $\WMI_{\mathbb{R}}$.} An example of a $\WMI$ problem with an SMT($\lra$) CNF formula $\theory$ over real variables $\X$ and Boolean variables $B$ and corresponding primal graph $\graph_{\theory}$ in (a). Their reductions to $\theory^{'}$ and $\graph_{\theory^{'}}$ as an $\WMI_{\mathbb{R}}$ problem in (b). The equivalent $\MI$ problem with formula $\theory''$ and primal graph $\graph_{\theory^{''}}$ over only real variables $\X''=\X\cup\{Z_{B},Z_{X_{1}},Z_{X_{2}}\}$ after the introduction of auxiliary variables $Z_{B},Z_{X_{1}},Z_{X_{2}}$. Note that $\graph_{\theory}$ and $\graph_{\theory}^{''}$ have the same treewidth one.}
\label{fig:wmi-mi}
\end{figure*}

Figure~\ref{fig:wmi-mi} illustrates one example of a reduction of a $\WMI$ problem to one $\WMIR$ one to a $\MI$ problem.
Consider the $\WMI$ problem over formula $\theory= (0<X_1<2)\wedge (0<X_2<2) \wedge (X_1+X_2<1) \wedge ( B\vee X_{1} > 1)$ on variables $\X=\{X_{1},X_{2}\},\B=\{B\}$ whose primal graph $\graph_{\theory}$ is also shown in Figure~\ref{fig:wmi-a}. 
Assume a weight function which decomposes as $w(X_{1}, X_{2}, B) =w_{\Gamma_{3}}(X_{1}, X_{2})w_{\Gamma_{4}}(X_{1}, B) = w_{\Gamma_{3}}(X_{1}, X_{2})w_{\Gamma_{4}}(X_{1})w_{\Gamma_{4}}(B)$ and whose values are $w_{\Gamma_{3}}(X_{1}, X_{2})=X_{1}X_{2}$, $w_{\Gamma_{4}}(X_{1})=2$ and $w_{\Gamma_{4}}(B)=3$ when $B$ is true and $w(B)=1$ otherwise.
The WMI of formula $\theory$ is:
\begin{align}
    \label{eq:wmi-ex-app}
    \WMI(\theory, w; \X, B) = &\int_0^1 ~d x_{1} \int_0^{2 - x_1} 1\times3x_{1}x_{2} ~d x_{2}\\\nonumber
    + &\int_1^2 ~d x_{1} \int_0^{2 - x_1} 2\times3x_{1}x_{2} ~d x_{2}\\\nonumber
    + &\int_1^2 ~d x_{1} \int_0^{2 - x_1} 2\times1x_{1}x_{2} ~d x_{2}~.
\end{align}

In Figure~\ref{fig:wmi-b}, we show the reduction to the above example problem to a $\WMIR$ one.
A free real variable $Z_{B}$ is introduced to replace Boolean variable $B$.
Then, the equivalent problem to the $\WMI$ one in Equation~\ref{eq:wmi-ex-app}, can be computed as:
\begin{align}
\label{eq:wmir-ex}
    \WMIR(\theory', w') = &\int_{0}^{1}~dz_{B}\int_0^1 ~d x_{1} \int_0^{2 - x_1} 1\times3x_{1}x_{2} ~d x_{2}\\\nonumber
    + &\int_{0}^{1}~dz_{B}\int_1^2 ~d x_{1} \int_0^{2 - x_1} 2\times3x_{1}x_{2} ~d x_{2}\\\nonumber
    + &\int_{-1}^{0}~dz_{B}\int_1^2 ~d x_{1} \int_0^{2 - x_1} 2\times1x_{1}x_{2} ~d x_{2} ~.
\end{align}

Figure~\ref{fig:wmi-c} illustrates the additional reduction from the above $\WMIR$ problem to a $\MI$ one.
There, additional real variables $Z'_{X_{1}}$, $Z_{X_{2}}$ and $Z''_{X_{1}}$ are added to formula $\theory''$ in substitution of the monomial weights attached to literal $\Gamma_{3}$ and $\Gamma_{4}$, respectively.
Therefore, the same result as Equation~\ref{eq:wmi-ex-app} and Equation~\ref{eq:wmir-ex} can be obtained as
\newpage
\begin{align}
    \MI(\theory'') = &\int_{0}^{3}~dz_{B}\int_0^1 ~d x_{1} \int_0^{2 - x_1}d x_{2}\int_{0}^{x_{1}} dz^{'}_{X_{1}}\int_{0}^{x_{2}} dz_{X_{2}}\\\nonumber
    &+ \int_{0}^{2}dz^{''}_{X_{1}} \int_{0}^{3}~dz_{B}\int_1^2 ~d x_{1} \int_0^{2 - x_1}d x_{2}\int_{0}^{x_{1}} dz^{'}_{X_{1}}\int_{0}^{x_{2}} dz_{X_{2}}\\\nonumber
    &+ \int_{0}^{2}dz^{''}_{X_{1}}\int_1^2 ~d x_{1} \int_0^{2 - x_1}d x_{2}\int_{0}^{x_{1}} dz^{'}_{X_{1}}\int_{0}^{x_{2}} dz_{X_{2}}~.
    \label{eq:mi-ex}
\end{align}

\section{Proofs}

\subsection{THEOREM 1 (MI of a formula with tree primal graph with unbounded diameter is \#P-Hard)}
\begin{proof}{(Theorem 1)}
We prove our complexity result by reducing a \#P-complete variant of the subset sum problem~\cite{garey2002computers} to an MI problem over an \smtlra~formula $\theory$ with tree primal graph whose diameter is $\bigO(n)$.
This problem is a counting version of subset sum problem saying that given a set of positive integers $S = \{s_1, s_2, \cdots, s_n \}$, and a positive integer $L$, and the goal is to count the number of subsets $S^\prime \subseteq S$ such that the sum of all the integers in the subset $S^\prime$ equals to $L$.

First, we reduce the counting subset sum problem in polynomial time to a model integration problem by constructing the following \smtlra~ formula $\theory$ on real variables $\X$ whose primal graph is shown in Figure~\ref{fig: linear no y_new}:

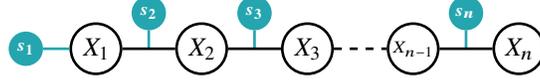
\begin{figure}[!tp]
\centering
\begin{tikzpicture}[grow=right]
\varnode[]{x1}{$X_1$};
\varnode[right=\middist of x1]{x2}{$X_2$};
\varnode[right=\middist of x2]{x3}{$X_3$};
\varnode[right=\middist of x3]{x4}{{\tiny$X_{n-1}$}};
\varnode[right=\middist of x4]{x5}{$X_{n}$};

\facnode[above=\tinydist of $ (x1) !.5! (x2) $]{s1}{$\boldsymbol{s_{2}}$};
\facnode[above=\tinydist of $ (x2) !.5! (x3) $]{s2}{$\boldsymbol{s_{3}}$};
\facnode[above=\tinydist of $ (x4) !.5! (x5) $]{s4}{$\boldsymbol{s_{n}}$};
\facnode[left=\middist of $(x1)$]{s0}{$\boldsymbol{s_1}$};

\path[-,draw=\faccolor,line width=\midlinewidth] (s1) -- ($ (x1) !.5! (x2) $);
\path[-,draw=\faccolor,line width=\midlinewidth] (s2) -- ($ (x2) !.5! (x3) $);
\path[-,draw=\faccolor,line width=\midlinewidth] (s4) -- ($ (x4) !.5! (x5) $);
\path[-,draw=\faccolor,line width=\midlinewidth] (s0) -- ($(s0) !.625! (x1)$);

\edge[] {x1}{x2};
\edge[] {x2}{x3};
\edge[dashed] {x3}{x4};
\edge[] {x4}{x5};

\end{tikzpicture}
\caption{Primal graph $\graph_{\theory}$ used for the \#P-hardness reduction in Theorem~6. We construct the corresponding formula $\theory$ such that $\graph_{\theory}$ has maximum diameter (it is a chain).
We graphically augment graph $\graph_{\theory}$ by introducing blue nodes to indicate that integers $s_i$ in set $S$ are contained in clauses between two variables.}
\label{fig: linear no y_new}
\end{figure}

\begin{align*}
    \Delta = 
    \left\{
    \begin{array}{lr}
         s_1 - \frac{1}{2n} < X_1 < s_1 + \frac{1}{2n} \lor -\frac{1}{2n} < X_1 < \frac{1}{2n}\\
         X_{i-1} + s_i - \frac{1}{2n} < X_i < X_{i-1} + s_i + \frac{1}{2n} ~\lor~         
         X_{i-1} - \frac{1}{2n} < X_i < X_{i-1} + \frac{1}{2n},\quad i = 2, \cdots n \\
    \end{array}
    \right.
\end{align*}

For brevity, we denote the first and the second literal in the $i$-th clause by $\lit(i,0)$ and $\lit(i,1)$ respectively.
Also We choose two constants $l = L - \frac{1}{2}$ and $u = L + \frac{1}{2}$.

In the following, we prove that
$n^n \MI(\Delta \land (l < y < u))$ equals to the number of subset $S^\prime \subseteq S$ whose element sum equals to $L$, which indicates that model integration problem whose tree primal graph has diameter $\bigO(n)$ is \#P-hard.

Let $\vv{a}^k = (a_1, a_2, \cdots, a_k)$ be some assignment to Boolean variables $(A_1, A_2, \cdots, A_k)$ with $a_i \in \{0, 1\}$, $i \in [k]$. Given an assignment $\vv{a}^k$, we define subset sums to be $S(\vv{a}^k) \triangleq \sum_{i=1}^k a_i s_i $, and formulas $\Delta_{\vv{a}^k} \triangleq \bigwedge_{i=1}^k \lit(i, a_i) $.

\begin{cla}\label{cla: assignment interval}
    The model integration for formula $\Delta_{\vv{a}^k}$ with an given assignment $\vv{a}^k \in \{0, 1\}^k$ is
    $\MI(\Delta_{\vv{a}^k}) = (\frac{1}{n})^{k}$. Moreover, for each variable $X_i$ in $\Delta_{\vv{a}^k}$,
    its satisfying assignments consist of the interval $[\sum_{j=1}^i a_j s_j - \frac{i}{2n}, \sum_{j=1}^i a_j s_j + \frac{i}{2n} ]$. Specifically, the satisfying assignments for variable $X_n$ in formula $\Delta_{\vv{a}^n}$ can be denoted by the interval $[S(\vv{a}^n) - \frac{1}{2}, S(\vv{a}^n) + \frac{1}{2}]$.
\end{cla}

\begin{proof}{(Claim~\ref{cla: assignment interval})}
    First we prove that $\MI(\Delta_{\vv{a}^k}) = (\frac{1}{n})^k$. For brevity, denote $a_i s_i$ by $\hat{s_i}$. By definition of model integration and the fact that the integral is absolutely convergent (since we are integrating a constant function, i.e., one, over finite volume regions), we have the following equation.
    \begin{align*}
        \MI(\Delta_{\vv{a}^k}) &=
        \int_{(x_1, \cdots, x_k) \models \Delta_{\vv{a}^k}} 1~\dif x_1 \cdots \dif x_k
        =
        \int_{\hat{s}_1-\frac{1}{2n}}^{\hat{s}_1+\frac{1}{2n}} \dif x_1 \cdots \int_{x_{k-2} + \hat{s}_{k-1}-\frac{1}{2n}}^{x_{k-2} + \hat{s}_{k-1}+\frac{1}{2n}} \dif x_{k-1}
        \int_{x_{k-1} + \hat{s}_k-\frac{1}{2n}}^{x_{k-1} + \hat{s}_k+\frac{1}{2n}} 1 ~\dif x_k 
    \end{align*}
    Observe that for the most inner integration over variable $x_k$, the integration result is $\frac{1}{n}$. By doing this iteratively, we have that $\MI(\theory_{\vv{a}^k}) = (\frac{1}{n})^k$.
    
    Next we prove that satisfying assignments for variable $X_i$ in formula $\Delta_{\vv{a}^k}$ is the interval $[\sum_{j=1}^i a_j s_j - \frac{i}{2n}, \sum_{j=1}^i a_j s_j + \frac{i}{2n} ]$ by mathematical induction.
    For $i=1$, since $X_1$ is in interval $[a_1 s_1 - \frac{1}{2n}, a_1 s_1 + \frac{1}{2n}]$, the statement holds in this case.
    Suppose that the statement holds for $i = m$, i.e. variable $X_m$ has its satisfying assignments in interval $[\sum_{j=1}^m a_j s_j - \frac{m}{2n}, \sum_{j=1}^m a_j s_j + \frac{m}{2n} ]$.
    Since variable $X_{m+1}$ has its satisfying assignments in interval $[X_m + a_{m+1} s_{m+1} - \frac{1}{2n}, X_m + a_{m+1} s_{m+1} + \frac{1}{2n}]$,
    then its satisfying assignments consist interval
    $[\sum_{j=1}^{m+1} a_j s_j - \frac{m+1}{2n}, \sum_{j=1}^{m+1} a_j s_j + \frac{m+1}{2n} ]$, that is, the statement also holds for $i = m+1$. Thus the statement holds.
    
\end{proof}

The above claim shows how to compute the model integration of formula $\theory_{\vv{a}^k}$. We will show in the next claim how to compute the model integration of formula $\theory_{\vv{a}^n}$ conjoined with a query $l < X_n < u$.

\begin{cla}\label{cla: explicit MI of delta with bdds}
For each assignment $\vv{a}^n \in \{0, 1\}^n$, the model integration of formula $\Delta_{\vv{a}^n} \land (l < X_n < u)$ falls into one of the following cases:
\begin{itemize}
    \item If $S(\vv{a}^n) < L$ or $S(\vv{a}^n) > L$, it holds that $\MI(\Delta_{\vv{a}^n}  \land (l < X_n < u) ) = 0$.
    \item If $S(\vv{a}^n) = L$, it holds that $\MI(\Delta_{\vv{a}^n}  \land (l < X_n < u) ) = (\frac{1}{n})^n$. 
\end{itemize}
\end{cla}

\begin{proof}{(Claim~\ref{cla: explicit MI of delta with bdds})}
    From the previous Claim~\ref{cla: assignment interval}, it is shown that variable $X_n$ has its satisfying assignments in interval $[S(\vv{a}^n) - \frac{1}{2}, S(\vv{a}^n) + \frac{1}{2}]$ in formula $\Delta_{\vv{a}^n}$ for each $\vv{a}^n \in \{0,1\}^n$.
    If $S(\vv{a}^n) < L$, given that $S(\vv{a}^n)$ is a sum of positive integers, then it holds that $S(\vv{a}^n) + \frac{1}{2} \leq (L-1) + \frac{1}{2} = L - \frac{1}{2} = l$ and therefore, $\MI(\Delta_{\vv{a}^n} \land (l < X_n < u)) = 0$;
    similarly, if $S(\vv{a}^n) > L$, then it holds that $S(\vv{a}^n) - \frac{1}{2} \geq u$ and therefore, $\MI(\Delta_{\vv{a}^n} \land (l < X_n < u)) = 0$.
    If $S(\vv{a}^n) = L$, by Claim~\ref{cla: assignment interval} we have that the satisfying assignment interval is inside the interval $[l, u]$ and thus it holds that $\MI(\Delta_{\vv{a}^n} \land (l < X_n < u)) = \MI(\Delta_{\vv{a}^n}) = (\frac{1}{n})^n$.
\end{proof}

In the next claim, we show how to compute the model integration of formula $\theory$ as well as for formula $\theory$ conjoined with query $l < X_n < u$ based on the already proven Claim~\ref{cla: assignment interval} and Claim~\ref{cla: explicit MI of delta with bdds}.

\begin{cla}\label{cla: MI of delta}
    The following two equations hold:
    \begin{enumerate}
        \item $\MI(\Delta) = \sum_{\vv{a}^n} \MI(\Delta_{\vv{a}^n})$.
        \item $\MI(\Delta \land (l < X_n < u)) = \sum_{\vv{a}^n} \MI(\Delta_{\vv{a}^n}  \land (l < X_n < u) )$.
    \end{enumerate}
\end{cla}
\begin{proof}{(Claim~\ref{cla: MI of delta})}
    Observe that for each clause in $\Delta$, literals are mutually exclusive since each $s_i$ is a positive integer.
    Then we have that formulas $\Delta_{\vv{a}^n}$ are mutually exclusive and meanwhile $\Delta = \bigvee_{\vv{a}^n} \Delta_{\vv{a}^n}$. Thus it holds that $\MI(\Delta) = \sum_{\vv{a}^n} \MI(\Delta_{\vv{a}^n})$.
    Similarly, we have formulas $(\Delta_{\vv{a}^n} \land (l < X_n < u))$'s are mutually exclusive and meanwhile $\Delta \land (l < X_n < u) = \bigvee_{\vv{a}^n} \Delta_{\vv{a}^n} \land (l < X_n < u)$. Thus the second equation holds.
\end{proof}

From the above claims, we can conclude that $\MI(\Delta \land (l < X_n < u)) = t (\frac{1}{n})^n$ where $t$ is the number of assignments $\vv{a}^n$ s.t. $S(\vv{a}^n) = L$.
Notice that for each $\vv{a}^n \in \{0, 1\}^n$, there is a one-to-one correspondance to a subset $S^\prime \subseteq S$ 
by defining $\vv{a}^n$ as $a_i = 1$ if and only if $s_i \in S^\prime$;
and $S(\vv{a}^n)$ equals to $L$ if and only if the sum of elements in $S^\prime$is $L$.
Therefore
$n^n \MI(\Delta \land (l < X_n < u))$ equals to the number of subset $S^\prime \subseteq S$ whose element sum equals to $L$. 

This finishes the proof for the statement that a model integration problem whose tree primal graph has diameter $\bigO(n)$ is \#P-hard.

\end{proof}

\subsection{THEOREM 2 (MI of a formula with primal graph with logarithmic diameter and treewidth two is \#P-Hard)}

\begin{proof}{(Theorem 2)}
Again we prove our complexity result by reducing the \#P-complete variant of the subset sum problem \cite{garey2002computers} to an MI problem over an \smtlra~formula $\theory$ with  primal graph whose diameter is $\bigO(\log n)$ and treewidth two.
In the \#P-complete subset sum problem, we are given a set of positive integers $S = \{s_1, s_2, \cdots, s_n \}$, and a positive integer $L$. The goal is to count the number of subsets $S^\prime \subseteq S$ such that the sum of all the integers in $S^\prime$ equals $L$.
    
First, we reduce this problem in polynomial time to a model integration problem with the following \smtlra~formula $\theory$ where variables are real and $u$ and $l$ are two constants. Its primal graph is shown in Figure~\ref{figure: tree-width 2}.
Consider $n = 2^k$, $n, k \in \mathbb{N}$.

\begin{figure}[!tp]
\centering
\begin{tikzpicture}[grow=right]

\varnode[]{x11}{{\scriptsize$X_{1,1}$}};
\varnode[right=\middist of x11, yshift=-22pt]{x21}{{\scriptsize$X_{2,2}$}};
\varnode[right=\middist of x11, yshift=22pt]{x22}{{\scriptsize$X_{2,1}$}};

\facnode[right=10pt of x11]{s1}{$\boldsymbol{\sum s}$};

\varnode[right=40pt of x22, yshift=30pt]{xk1}{{\scriptsize$X_{k,1}$}};
\varnode[right=\middist of xk1, yshift=-22pt]{xk21}{{\scriptsize$X_{k+1,2}$}};
\varnode[right=\middist of xk1, yshift=22pt]{xk22}{{\scriptsize$X_{k+1,1}$}};

\facnode[right=10pt of xk1]{s2}{$\boldsymbol{\sum s}$};

\varnode[right=40pt of x21, yshift=-30pt]{xkn1}{{\scriptsize$X_{k,\frac{n}{2}}$}};
\varnode[right=\middist of xkn1, yshift=-22pt]{xkn21}{{\scriptsize$X_{k+1,n}$}};
\varnode[right=\middist of xkn1, yshift=22pt]{xkn22}{{\fontsize{5pt}{4pt}{$X_{k+1,n-1}$} }};

\facnode[right=10pt of xkn1]{s3}{$\boldsymbol{\sum s}$};

\facnode[right=\middist of xk22]{s11}{$\boldsymbol{s_1}$};
\facnode[right=\middist of xk21]{s12}{$\boldsymbol{s_2}$};
\facnode[right=12pt of xkn22]{s13}{$\boldsymbol{\tiny{s_{n-1}}}$};
\facnode[right=\middist of xkn21]{s14}{$\boldsymbol{s_n}$};

\edge[] {x11}{x21,x22};
\edge[] {x21}{x22};

\edge[] {xk1}{xk21,xk22};
\edge[] {xk21}{xk22};

\edge[] {xkn1}{xkn21,xkn22};
\edge[] {xkn21}{xkn22};

\path[dashed, draw, line width=\midlinewidth] (x22) -- ($ (x22) !.4! (xk1) $);
\path[dashed, draw, line width=\midlinewidth] (x22) -- ($ (x22) !.35! (xkn1) $);

\path[dashed, draw, line width=\midlinewidth] (x21) -- ($ (x21) !.4! (xkn1) $);
\path[dashed, draw, line width=\midlinewidth] (x21) -- ($ (x21) !.35! (xk1) $);

\path[dashed, draw, line width=\midlinewidth] (xk1) -- ($ (xk1) !.35! (x22) $);
\path[dashed, draw, line width=\midlinewidth] (xkn1) -- ($ (xkn1) !.35! (x21) $);

\path[dashed, draw, line width=\midlinewidth] (xk1) -- ($ (xk1) !.3! (xkn1) $);
\path[dashed, draw, line width=\midlinewidth] (xkn1) -- ($ (xkn1) !.3! (xk1) $);

\path[-,draw=\faccolor,line width=\midlinewidth] (s1) -- ($ (x11) !.5! (x21) $);
\path[-,draw=\faccolor,line width=\midlinewidth] (s1) -- ($ (x21) !.5! (x22) $);
\path[-,draw=\faccolor,line width=\midlinewidth] (s1) -- ($ (x11) !.5! (x22) $);

\path[-,draw=\faccolor,line width=\midlinewidth] (s2) -- ($ (xk1) !.5! (xk21) $);
\path[-,draw=\faccolor,line width=\midlinewidth] (s2) -- ($ (xk21) !.5! (xk22) $);
\path[-,draw=\faccolor,line width=\midlinewidth] (s2) -- ($ (xk1) !.5! (xk22) $);

\path[-,draw=\faccolor,line width=\midlinewidth] (s3) -- ($ (xkn1) !.5! (xkn21) $);
\path[-,draw=\faccolor,line width=\midlinewidth] (s3) -- ($ (xkn21) !.5! (xkn22) $);
\path[-,draw=\faccolor,line width=\midlinewidth] (s3) -- ($ (xkn1) !.5! (xkn22) $);

\path[-,draw=\faccolor,line width=\midlinewidth] (s11) -- ($(s11) !.7! (xk22)$);
\path[-,draw=\faccolor,line width=\midlinewidth] (s12) -- ($(s12) !.7! (xk21)$);
\path[-,draw=\faccolor,line width=\midlinewidth] (s13) -- ($(s13) !.6! (xkn22)$);
\path[-,draw=\faccolor,line width=\midlinewidth] (s14) -- ($(s14) !.7! (xkn21)$);

\end{tikzpicture}
\caption{Primal graph used for \#P-hardness reduction in Theorem 7. We also put blue nodes to indicate that integer $s_i$'s in set $S$ are contained in some clauses and that model integration over some cliques is the sum of some $s_i$'s.}
\label{figure: tree-width 2}
\end{figure}
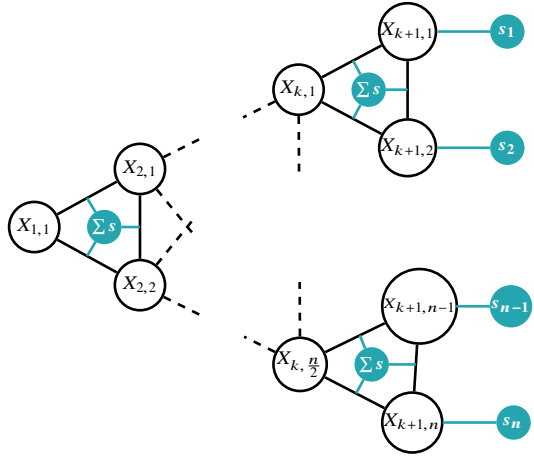

\begin{align*}
    & \theory = \bigwedge_{i \in [n]} (- \frac{1}{4n} < X_{k+1, i} < \frac{1}{4n} \lor - \frac{1}{4n} + s_i < X_{k+1, i} < \frac{1}{4n} + s_i) 
    \bigwedge \theory_t \\
    & \text{where}~\theory_t = \bigwedge_{j \in [k], i \in [2^j] } - \frac{1}{4n} + X_{j+1, 2i-1} + X_{j+1, 2i} < X_{j, i} < \frac{1}{4n} + X_{j+1, 2i-1} + X_{j+1, 2i}
\end{align*}

For brevity, we denote all the variables by $\X$ and denote the literal $- \frac{1}{4n} < X_{k+1,i} < \frac{1}{4n}$ by $\lit(i, 0)$ and literal $- \frac{1}{4n} + s_i < X_{k+1,i} < \frac{1}{4n} + s_i$ by $\lit(i, 1)$ respectively.
Also We choose two constants $l = L - \frac{1}{2}$ and $u = L + \frac{1}{2}$.
In the following, we prove that
$(2n)^{2n-1} \MI(\Delta \land (l < X_{1,1} < u))$ equals to the number of subset $S^\prime \subseteq S$ whose element sum equals to $L$, which indicates that model integration problem with primal graph whose diameter is $\bigO(\log n)$ and treewidth two is \#P-hard.
    
Let $\vv{a}^n = (a_1, a_2, \cdots, a_n) \in \{0,1\}^n$ be some assignment to Boolean variables $(A_1, A_2, \cdots, A_n)$. Given an assignment $\vv{a}^n$, define the sum as $S(\vv{a}^n) \triangleq \sum_{i=1}^n a_i s_i$, and formula as $\Delta_{\vv{a}^n} \triangleq \bigwedge_{i=1}^n \lit(i, a_i) \land \theory_t$.
    
\begin{cla}\label{cla: assignment interval 2}
    The model integration for formula $\Delta_{\vv{a}^n}$ with given $\vv{a}^n \in \{0, 1\}^n$ is
    $\MI(\Delta_{\vv{a}^n}) = (\frac{1}{2n})^{2n - 1}$.
    Moreover, for each variable $X_{j,i}$ in formula $\Delta_{\vv{a}^n}$, its satisfying assignments consist of the interval $[\sum_l a_l s_l - \frac{2^{k-j+2}-1}{4n},\sum_l a_l s_l + \frac{2^{k-j+2}-1}{4n}]$ where $l \in \{ l \mid X_{k+1, l} \text{ is a descendant of } X_{j,i} \}$.
    Specifically, the satisfying assignments for the root variable $X_{1,1}$ can be denoted the interval $[S(\vv{a}^n) - \frac{2n-1}{4n}, S(\vv{a}^n) + \frac{2n-1}{4n}] \subset [S(\vv{a}^n) - \frac{1}{2}, S(\vv{a}^n) + \frac{1}{2}]$.
\end{cla}

\begin{proof}{(Claim~\ref{cla: assignment interval 2})}

First we prove that $\MI(\Delta_{\vv{a}^n}) = (\frac{1}{2n})^{2n - 1}$. For brevity, denote $a_i s_i$ by $\hat{s_i}$. By definition of model integration and the fact that the integral is absolutely convergent (since we are integrating a constant function, i.e., one, over finite volume regions), 
we have the following equations
\begin{align*}
        \MI(\Delta_{\vv{a}^n}) 
        &=
        \int_{\x \models \Delta_{\vv{a}^n}} 1~\dif \X \\
        &= 
        \int_{-\frac{1}{4n} + \hat{s}_n}^{\frac{1}{4n} + \hat{s}_n} \dif x_{k+1, n}
        \cdots
        \int_{-\frac{1}{4n} + \hat{s}_1}^{\frac{1}{4n} + \hat{s}_1} \dif x_{k+1, 1}
        \int_{- \frac{1}{4n} + x_{k+1, n-1} + x_{k+1, n}}^{\frac{1}{4n} + x_{k+1, n-1} + x_{k+1, n}} \dif x_{k, 2^{k-1}} 
        \cdots
        \int_{- \frac{1}{4n} + x_{2, 1} + x_{2, 2}}^{\frac{1}{4n} + x_{2, 1} + x_{2, 2}} 1 ~\dif x_{1,1} ~.
\end{align*}

Observe that for the most inner integration over variable $x_{1,1}$, the integration result is $\frac{1}{2n}$. By doing this iteratively, we have that $\MI(\theory_{\vv{a}^k}) = (\frac{1}{2n})^{2n-1}$ where the $2n-1$ comes from the number of variables.

Then we prove that satisfying assignments for variable $X_{j,i}$ in formula $\Delta_{\vv{a}^n}$ lie in the interval $[\sum_l a_l s_l - \frac{2^{k-j+2}-1}{4n},\sum_l a_l s_l + \frac{2^{k-j+2}-1}{4n}]$ where $l \in \{ l \mid X_{k+1, l} \text{ is a descendant of } X_{j,i} \}$ by performing mathematical induction in a bottom-up way.

For $j=1$, any variable $X_{k+2-j,i}$ with $i \in [2^{k+2-j}]$ has satisfying assignments consisting of the interval $[a_i s_i - \frac{1}{4n}, a_i s_i + \frac{1}{4n}]$. Thus the statement holds for this case.

Suppose that the statement holds for $j=m$, that is, for any $i \in [2^{k+2-m}]$, any variable $X_{k+2-m,i}$ has satisfying assignments consisting interval $[\sum_l a_l s_l - \frac{2^m - 1}{4n},\sum_l a_l s_l + \frac{2^{m}-1}{4n}]$ where $l \in \{ l \mid X_{k+1, l} \text{ is a descendant of } X_{k+2-m,i} \}$.

Then for $j=m+1$ and any $i \in [2^{k+1-m}]$, the variable $X_{k+1-m,i}$ has two descendants, variable $X_{k+2-m,2i-1}$ and variable $X_{k+2-m,2i}$.
Moreover, we have that $-\frac{1}{4n} + X_{k+2-m,2i-1} + X_{k+2-m,2i} < X_{k+1-m,i} < \frac{1}{4n} + X_{k+2-m,2i-1} + X_{k+2-m,2i}$. Then the lower bound of the interval for variable $X_{k+1-m,i}$ is $-\frac{1}{4n} + \sum_l a_l s_l - 2\frac{2^m - 1}{4n} = \sum_l a_l s_l - \frac{2^{m+1} - 1}{4n}$; similarly the upper bound of the interval is $\sum_l a_l s_l + \frac{2^{m+1} - 1}{4n}$, where $l \in \{ l \mid X_{k+1,l} \text{ is a descendant of } X_{k+1-m,i} \}$. That is, the statement also holds for $j = m+1$ which finishes our proof.
\end{proof}

The above claim shows what the model integration of formula $\theory_{\vv{a}^k}$ is like. We'll show in the next claim what the model integration of formula $\theory_{\vv{a}^n}$ conjoined with a query $l < X_{1,1} < u$ is like.

\begin{cla}\label{cla: explicit MI of delta with bdds 2}
    For each assignments $\vv{a}^n \in \{0, 1\}^n$, the model integration of
    $\Delta_{\vv{a}^n} \land (l < X_{1,1} < u)$ falls into one of the following cases:
    \begin{itemize}
        \item If $S(\vv{a}^n) < L$ or $S(\vv{a}^n) > L$, then $\MI(\Delta_{\vv{a}^n} \land (l < X_{1,1} < u)) = 0$.
        \item If $S(\vv{a}^n) = L$, then $\MI(\Delta_{\vv{a}^n} \land (l < X_{1,1} < u)) = (\frac{1}{2n})^{2n - 1}$.
    \end{itemize}
\end{cla}

\begin{proof}{(Claim~\ref{cla: explicit MI of delta with bdds 2})}
From previous Claim~\ref{cla: assignment interval 2}, it is shown that variable $X_{1,1}$ has its satisfying assignments in the interval $[S(\vv{a}^n) - \frac{2n - 1}{4n}, S(\vv{a}^n) + \frac{2n - 1}{4n}]$ in formula $\Delta_{\vv{a}^n}$ for each $\vv{a}^n \in \{0,1\}^n$.

If $S(\vv{a}^n) < L$, given that $S(\vv{a}^n)$ is a sum of positive integers, then it holds that $S(\vv{a}^n) + \frac{1}{2} \leq (L-1) + \frac{2n - 1}{4n} < L - \frac{1}{2} = l$ and therefore, $\MI(\Delta_{\vv{a}^n} \land (l < X_{1,1} < u)) = 0$;
similarly, if $S(\vv{a}^n) > L$, then it holds that $S(\vv{a}^n) - \frac{1}{2} > u$ and therefore, $\MI(\Delta_{\vv{a}^n} \land (l < X_{1,1} < u)) = 0$.
If $S(\vv{a}^n) = L$, then by Claim~\ref{cla: assignment interval 2} we have that the satisfying assignment interval is inside the interval $[l, u]$ and thus it holds that $\MI(\Delta_{\vv{a}^n} \land (l < X_{1,1} < u)) = \MI(\Delta_{\vv{a}^n}) = (\frac{1}{2n})^{2n-1}$.
\end{proof}

\begin{cla}\label{cla: MI of delta 2}
    The following two equations hold:
    \begin{enumerate}
        \item $\MI(\Delta) = \sum_{\vv{a}^n} \MI(\Delta_{\vv{a}^n})$.
        \item $\MI(\Delta \land (l < X_{1,1} < u)) = \sum_{\vv{a}^n} MI(\Delta_{\vv{a}^n} \land (l < X_{1,1} < u))$.
    \end{enumerate}
\end{cla}
\begin{proof}{(Claim~\ref{cla: MI of delta 2})}
    Observe that for each pair of literals $\lit(i,0)$ and $\lit(i,1), i \in [n]$, literals are mutually exclusive since each $s_i$ is a positive integer.
    Then we have that formulas $\Delta_{\vv{a}^n}$ are mutually exclusive and meanwhile formula $\Delta = \bigvee_{\vv{a}^n} \Delta_{\vv{a}^n}$. Thus it holds that $\MI(\Delta) = \sum_{\vv{a}^n} \MI(\Delta_{\vv{a}^n})$.
    Similarly, we have formulas $(\Delta_{\vv{a}^n} \land (l < X_{1,1} < u))$'s are mutually exclusive and meanwhile $\Delta \land (l < X_{1,1} < u) = \bigvee_{\vv{a}^n} \Delta_{\vv{a}^n} \land (l < X_{1,1} < u)$. Thus the second equation holds.
\end{proof}

From the above claims, we can conclude that $\MI(\Delta \land (l < X_{1,1} < u)) = t (\frac{1}{2n})^{2n-1}$ where $t$ is the number of assignments $\vv{a}^n$ s.t. $S(\vv{a}^n) = L$.
Notice that for each $\vv{a}^n \in \{0, 1\}^n$, 
there is a one-to-one correspondence to a subset $S^\prime \subseteq S$ by defining $\vv{a}^n$ as $a_i = 1$ if and only if $s_i \in S^\prime$;
and $S(\vv{a}^n)$ equals to $L$ if and only if the sum of elements in $S^\prime$ is $L$.
Therefore 
$(2n)^{2n-1} \MI(\Delta \land (l < X_{1,1}  < u))$ equals to the number of subset $S^\prime \subseteq S$ whose element sum equals to $L$. 

This finishes the proof for the statement that 
a model integration problem with primal graph whose diameter is $\bigO(\log n)$ and treewidth two is \#P-hard.
\end{proof}

\subsection{PROPOSITION 3 (MI via message passing)}
\begin{proof}{(Proposition 3)}

By the definition of downward pass beliefs and messages, we have that the downward pass belief $\bel_{i^*}$ of a node $i^*$ can be written as follows
\begin{align*}
    \bel_{i^*}(x_{i^*}) 
    &=\prod_{j \in \neigh({i^*})} \msg{j}{i^*}{}(x_{i^*}) 
    =\prod_{j \in \neigh({i^*})} \int_{\mathbb{R}} \id{x_{i^*},x_{j}\models\theory_{i^*,j}}\id{x_j\models\theory_{j}} \prod_{c\in\neigh({j})\setminus\{{i^*}\}}\msg{c}{j}{}(x_j) ~d x_{j} \\
    &= \int_{\mathbb{R}^{ \mid \neigh({i^*}) \mid }} \prod_{({i^*}, j) \in \E } \id{x_{i^*},x_{j}\models\theory_{{i^*},j}}\id{x_j\models\theory_{j}} \prod_{(j, c) \in \E, c \neq {i^*}  }\msg{c}{j}{}(x_j) ~d \x_{i^*}~,
\end{align*}
where the last equality comes from interchanging integration with product, and $\x_{i^*}$ is defined as $\x_{i^*} = \{ x_j \mid (i^*, j) \in \E \}$.
By doing this recursively, i.e. plugging in the messages as defined in Equation 7, the belief of node $i^*$ can be expressed as follows
\begin{align*}
    \bel_{i^*}(x_{i^*}) 
    &= \int_{\mathbb{R}^{|\X|-1}} \prod_{(i^*, j) \in \E} \id{x_{i^*},x_{j}\models\theory_{i^*,j}}
    \prod_{j \in \Va \setminus\{i^*\}}\id{x_j\models\theory_{j}} ~d \x\setminus\{x_{i^*}\} =\int_{\mathbb{R}^{|\X|-1}}\id{\x\models\theory} ~d \x\setminus\{x_{i^*}\}~.  
\end{align*}
The last equality comes from the fact that the formula $\theory$ has a tree primal graph $\graph_{\theory}$, i.e. $\theory = \land_{(i, j) \in \E} \theory_{i, j} \land_{i \in \Va} \theory_i$.
Recall the definition of MI as defined in Equation 6, we have that the final belief $\bel_{i^*}$ is the unnormalized marginal of variable $X_{i^*} \in \X$, i.e. $\bel_{i^*}(x_{i^*})=p_{\theory}(x_{i^*})\cdot \MI(\theory)$.
Besides, this also indicates that the integration over the belief of $X_{i^*}$ is equal to the MI of formula $\theory$.
\end{proof}

\subsection{PROPOSITION 4 (Messages and beliefs)}
\begin{proof}{(Proposition 4)}

This follows by induction on both the level of the node and the number of its neighbors. Consider the base case of a node $i$ with only one neighbor $j$ being the leaf node.
Then the message sent from node $j$ to node $i$ would be
$\msg{i}{j}{}(x_j) = \int_{\mathbb{R}} \id{x_{i}, x_{j}\models\theory_{i,j}}\id{x_i\models\theory_{i}}\ d x_{i}$.
This integral has one as an integrand over pieces that satisfy the logical constraints $\id{x_{i}, x_{j}\models\theory_{i,j}}\id{x_i\models\theory_{i}}$ with integration bounds linear in variable $x_j$.
Therefore the resulting message from node $j$ to node $i$ is a piecewise linear function in variable $x_j$.
Since node $i$ has only one child by assumption, its upward-pass belief is also piecewise univariate polynomial.

From here, the proof follows for any message and belief for more complex tree structures by considering that the piecewise polynomial family is closed under multiplication and integration.
\end{proof}

\subsection{PROPOSITION 5 (Univariate queries via message passing)}
\begin{proof}{(Proposition 5)}

For an SMT($\lra$) query $\Phi$ over a variable $X_{i}\in\X$, the MI over formula $\theory$ conjoined with query $\Phi$ can be expressed as follows by the definition of model integration.
\begin{align*}
    \MI(\theory\wedge\Phi)=\int_{\mathbb{R}^{|\X|}}\id{\x\models\theory \wedge\Phi } ~d \x 
    =\int_{\mathbb{R}^{|\X|}}\id{x_{i}\models\Phi} \id{\x\models\theory } ~d \x\setminus\{x_{i}\}d x_{i}~.
\end{align*}

Notice that by the proof of Proposition 2, we have that the downward pass belief of node $i$ is
$\bel_i(x_i) = \int_{\mathbb{R}^{|\X|-1}}\id{\x\models\theory} ~d \x\setminus\{x_{i}\}$.
By plugging the belief $\bel_i$ in the above equation of MI over formula $\theory \land \Phi$,
we have that
\begin{align*}    
    \MI(\theory\wedge\Phi)
    =\int_{\mathbb{R}}\id{x_{i}\models\Phi}\id{x_{i}\models\theory_{i}}\bel_{i}(x_i) d x_{i}~.
\end{align*}
\end{proof}

\subsection{PROPOSITION 6 (Bivariate queries via message passing)}
\begin{proof}{(Proposition 6)}

Denote the \smtlra~formula $\theory \land \Phi$ by $\theory^*$ where $\Phi$ is an SMT($\lra$) query over variables $X_{i^*}, X_{j^*}\in\X$. We also denote the belief and messages in formula $\theory^*$ by $\bel_i^*$ and $\msg{i}{j}{*}$ respectively.

Notice that since query $\Phi$ is defined over variables $X_{i^*}, X_{j^*}$, then it holds that 
for any $(i, j) \in \E$, $\theory_{i, j} = \theory_{i, j}^*$ if $(i, j) \neq (i^*, j^*)$; 
else $\theory_{i^*, j^*}^* = \theory_{i^*, j^*} \land \Phi$.
Also for any $i \in \Va$, it holds that $\theory_i = \theory_i^*$.
Therefore, we have that $\bel_{j^*}^*(x_{j^*}) / \msg{i^*}{j^*}{*} = \bel_{j^*}(x_{j^*}) / \msg{i^*}{j^*}{}$ by the definition of beliefs and messages.
Moreover, we can compute the message sent from node $j^*$ to node $i^*$ in formula $\theory^*$ as follows:
\begin{align*}
    \msg{j^*}{^*i}{*}(x_{i^*}) 
    &= \int_{\mathbb{R}} \bel^*_{j^*}(x_{j^*}) / \msg{i^*}{j^*}{*}(x_{j^*}) \times
    \id{x_{i^*}, x_{j^*}\models \theory_{i^*,j^*}^* }\id{x_{j^*}\models\theory_{j^*}} ~d x_{j^*} \\
    &= \int_{\mathbb{R}} \bel_{j^*}(x_{j^*}) / \msg{i^*}{j^*}{}(x_{j^*}) \times
    \id{x_{i^*}, x_{j^*}\models \theory_{i^*,j^*}\land\Phi }\id{x_{j^*}\models\theory_{j}} ~d x_{j^*}.
\end{align*}
Similarly, we have that the final belief on node 
$i^*$ is as follows:
\begin{align*}
    \bel^*_{i^*}(x_{i^*}) = \prod_{j \in \neigh(i^*)} \msg{j}{i^*}{*}(x_{i^*}) = \msg{i^*}{j^*}{*}(x_{i^*}) \prod_{j \in \neigh(i^*), j\neq j^*} \msg{j}{i^*}{}(x_{i^*}).   
\end{align*}
Then the MI over formula $\theory \land \Phi$ can be computed by doing $\MI(\theory \land \Phi) = \int_{\mathbb{R}} \id{x_{i^*} \models \theory_{i^*}} \cdot \bel_{i^*}^*(x_{i^*}) ~d x_{i^*}$
where messages except $\msg{i^*}{j^*}{*}$ are pre-computed and the computation of the message $\msg{i^*}{j^*}{*}$ can reuse the pre-computed beliefs as shown above.

\end{proof}

\subsection{PROPOSITION 7 (Statistical moments via message passing)}
\begin{proof}{(Proposition 7)}

By the definition of the k-th moment of the random variables and Proposition 2 that belief ~$\bel_{i}$~ of node $i$ is the unnormalized marginal $p_{i}(x_{i})$ of variable $X_i \in \X$, we have that
\begin{align*}
    \mathbb{E}[X_i^k] =\int_{\mathbb{R}}\id{x_{i}\models\theory_{i}} \times x_i^k p_{\theory}(x_{i}) ~d x_{i} 
    = \frac{1}{\MI(\theory)} \int_{\mathbb{R}}\id{x_{i}\models\theory_{i}} \times x_i^k \bel_{i}(x_i) ~d x_{i}~.
\end{align*}
\end{proof}

\end{document}